\documentclass[journal]{IEEEtran}

\usepackage{url}

\usepackage{algorithm}
\usepackage{algorithmic}
\usepackage{graphicx}
\usepackage{bbm}
\usepackage{amsfonts}
\usepackage{amsmath}
\usepackage{xcolor}
\usepackage{enumerate}
\usepackage{multirow}
\usepackage{booktabs}
\usepackage{subcaption}

\newtheorem{theorem}{Theorem}
\newtheorem{assumption}{Assumption}

\newtheorem{lemma}{Lemma}

\newtheorem{corollary}{Corollary}

\newcommand{\rev}[1]{{\color{black}#1}}
\newcommand{\AQrev}[1]{{\color{black}#1}}

\renewcommand{\appendixname}{Appendix}

\begin{document}

\title{Coverage Guarantees for Pseudo-Calibrated Conformal Prediction under Distribution Shift}

\author{Farbod Siahkali, \IEEEmembership{Graduate Student Member, IEEE}, Ashwin Verma, \IEEEmembership{Member, IEEE}, Vijay Gupta, \IEEEmembership{Fellow, IEEE}
\thanks{This work was supported in part by the U.S. Army Research Office under Grant 13001664. Authors are with the Elmore Family School of Electrical and Computer Engineering at Purdue University, West Lafayette, IN 47907 USA (e-mail: \texttt{\{siahkali,verma240,gupta869\}@purdue.edu}).}
}

\maketitle

\begin{abstract}
\rev{Conformal prediction (CP) provides distribution-free marginal coverage under exchangeability, but coverage can fail under distribution shift. We study pseudo-calibration for unlabeled target data under bounded label-conditional feature shift. Using domain-adaptation tools, we derive target coverage lower bounds in terms of source classifier loss and Wasserstein shift. We also analyze fixed slack inflation of the pseudo-calibrated threshold and use this result to motivate a heuristic threshold adjustment. Finally, we propose source-tuned pseudo-calibration, which interpolates between hard pseudo-labels and randomized labels based on classifier uncertainty. Experiments on MNIST, CIFAR-10, and CIFAR-100 show that the proposed method mitigates coverage degradation under shift, with larger expected set size.}
\end{abstract}

\IEEEpeerreviewmaketitle

\vspace{-0.4em}

\section{Introduction}

\IEEEPARstart{C}{onformal} prediction (CP) \rev{provides} a framework for constructing prediction sets with guaranteed marginal coverage under an exchangeability assumption between calibration and test data~\cite{angelopoulos2024theoretical, tibshirani2019conformal}. CP has been applied in a variety of settings~\cite{11267082,11077387,10528345, cpfusion}. However, the finite-sample, distribution-free guarantees from CP rely on exchangeability~\cite{barber2023conformal}, which is often violated due to distribution shift between the source (calibration) and target (test) domains~\cite{moreno2012unifying, podkopaev2021distribution,xu2025wassersteinregularized,correia2025nonexchangeableconformalpredictionoptimal}. Under covariate shift, one approach to correct such miscoverage is to utilize weighted CP by importance weighting the calibration scores with estimated density ratios between source and target marginals on the input space~\cite{tibshirani2019conformal, xu2025wassersteinregularized}. \rev{Related weighted conformal methods have also been used for counterfactual and individual treatment-effect inference under covariate shift and strong ignorability~\cite{lei2021conformal}.} Alternatively, robust distributional approaches construct ambiguity sets around the score distribution and propagate worst-case perturbations through the conformal quantile in score space~\cite{aolaritei2025conformal}.

When target labels are unavailable, pseudo-labels from \rev{a classifier can help, but errors} can degrade coverage. \rev{Unlabeled-target heuristics rescale} scores using entropy or reconstruction loss~\cite{kasa2025adaptingpredictionsetsdistribution,Alijani_2025_CVPR}. However, these methods do not yield analytical coverage guarantees. Alternatively,~\cite{pmlr-v266-angelman25a} offers bounds for pseudo-labeled targets using score distribution distances. Since these bounds ignore the classifier and shift characteristics, they do not provide insights on designing the classifier or pseudo-calibration schemes for trading off coverage and set size.

These limitations point to a broader gap: existing CP methods under distribution shift do not account for how source-domain \rev{classification errors translate to target-domain conformal coverage}. Domain adaptation (DA) theory provides a natural lens for this question by bounding target errors in terms of source losses and distributional shift measures~\cite{kumar2020understanding,he2024gradual}.
\rev{Yet, how source-domain classifier losses and label-conditional shift control unlabeled-target pseudo-calibrated conformal coverage remains underdeveloped.}
Our work bridges this gap by drawing on DA theory to derive \AQrev{multiclass} coverage guarantees \AQrev{for pseudo-calibration} that explicitly depend on classifier properties and shift measures.

\rev{Our setting differs from the settings considered in works that use weighted-~\cite{tibshirani2019conformal,lei2021conformal}, PAC-style-~\cite{park2022pac}, conformal-risk-control-~\cite{angelopoulos2022conformalrisk}, \AQrev{and} doubly robust-~\cite{paul2025multiply,yang2024doubly} conformal methods\AQrev{, as well as probability calibration via unsupervised domain adaptation~\cite{pmlr-v108-park20b}}. The works on covariate shift in~\cite{tibshirani2019conformal,lei2021conformal,park2022pac} study settings in which the label conditional is preserved, i.e., $P_{Y\mid X}=Q_{Y\mid X}$ a.e., and use the feature density ratio $dQ_X/dP_X$, i.e., the Radon--Nikodym derivative of $Q_X$ with respect to $P_X$ when $Q_X\ll P_X$, to reweight source calibration scores toward the target distribution.

We instead consider a setting motivated by label-preserving changes in sensing conditions or feature appearance, where unlabeled target inputs are available, the label marginals remain unchanged, and the class-conditional feature distributions $P_{X\mid Y}$ and $Q_{X\mid Y}$ may differ.
Consequently, the label conditionals may change, and density-ratio reweighting of source calibration scores is not generally sufficient to recover conformal coverage in this setting. We therefore construct the conformal threshold using pseudo-labeled target inputs and then derive a lower bound on target coverage whose deficit from the nominal level $1-\alpha$ is controlled by the source classifier loss and the label-conditional Wasserstein shift.}

Our contributions can be summarized as follows.
First, we derive coverage lower bounds for pseudo-calibrated prediction sets on the target domain in terms of the classifier’s source-domain loss, Lipschitz property, and Wasserstein measure of the distribution shift.
Second, we introduce relaxed pseudo-calibrated sets that inflate the conformal threshold by a slack parameter, \rev{derive the corresponding coverage lower bound for fixed slack values, and study a heuristic for choosing this slack in experiments.}
Finally, \AQrev{inspired by~\cite{candes2025probably},} we propose a source-tuned pseudo-calibration method that interpolates between hard pseudo-labels and randomized labels based on classifier uncertainty. \rev{We establish a monotonicity guarantee relative to hard pseudo-calibration, an explicit improvement of its coverage lower bound, and empirically observe improved target coverage.}

\section{Background and Problem Formulation}

\paragraph{Conformal Prediction} Given calibration data $\{(X_i, Y_i)\}_{i=1}^n \stackrel{\text{i.i.d.}}{\sim} P_{XY}$ and a test point $(X_{n+1}, Y_{n+1})$, the goal is to ensure
\begin{equation}
    \Pr \{ Y_{n+1} \in C^{1-\alpha}_P(X_{n+1}) \} \ge 1 - \alpha,
    \label{eq:cpcoverage}
\end{equation}
for any $\alpha \in (0,1)$ under an exchangeability assumption between the calibration and test point, i.e., their joint distribution is invariant under permutations~\cite[Section 3]{shafer2008tutorial}. Let $s: \mathcal{X} \times \mathcal{Y} \to \mathbb{R}$ be a nonconformity score. For a distribution \rev{$P$} on $\mathcal{X} \times \mathcal{Y}$, denote the pushforward score distribution by \rev{$s\#P$}, with CDF $F_{s\#P}$. For $\alpha\in (0,1)$, define ($1-\alpha$)-quantile as $\inf\{ t \in \mathbb{R} : F_{s\#P}(t) \ge 1-\alpha \}$.
With the empirical distribution $\hat{P}_n = \frac{1}{n} \sum_{i=1}^n \delta_{(X_i, Y_i)}$, the split-conformal threshold at $\alpha$ is
\begin{equation}
\label{eq:empquantile}
q_{P,\alpha}
:=
\mathrm{Quantile} \left(\frac{\big\lceil (1-\alpha)(n+1) \big\rceil}{n}; s\#\hat{P}_n \right),
\end{equation}
\AQrev{with $q_{P,\alpha}:=+\infty$ if $\lceil(n+1)(1-\alpha)\rceil=n+1$,} and the conformal set is given by
$
C^{1-\alpha}_P(x)
=
\big\{ y \in \mathcal{Y} : s(x, y) \leq q_{P,\alpha} \big\}$.
Under exchangeability, this construction guarantees~\eqref{eq:cpcoverage}~\cite[Theorem 1.1]{angelopoulos2022gentleintroductionconformalprediction}. However, \AQrev{coverage can degrade when test data are drawn from $Q_{XY}\neq P_{XY}$.}

\rev{\paragraph{Distribution shift measure} We assume in this paper that all probability measures considered are supported on the metric space $(\mathbb{R}^d,\|\cdot\|_2)$.}
For $p \ge 1$ and probability measures $P$ and $Q$ on $\mathcal X$, the $p$-Wasserstein distance is
\[
W_p(P, Q)
:=
\left(
  \inf_{\pi \in \Pi(P, Q)}
  \mathbb{E}_{(X,X') \sim \pi}\big[\|X - X'\|_2^p\big]
\right)^{1/p},
\]
where $\Pi(P, Q)$ is the set of all couplings of $P,Q$.
For $p = \infty$, let
$
W_\infty(P, Q)
:=
\inf_{\pi \in \Pi(P, Q)}
\operatorname*{ess\,sup}_{(X,X')\sim\pi}\|X - X'\|_2.
$

\paragraph{Problem Considered}
We consider a multiclass classification setting with input space ${\mathcal{X} = \mathbb{R}^d}$ and label space $\mathcal{Y} = [K] := \{1, \dots, K\}$. The source and target domains are represented by joint distributions $P_{XY}$ and $Q_{XY}$ over $\mathcal{X} \times \mathcal{Y}$. A classifier $f : \mathcal{X} \to \mathcal{Y}$ is induced by a logit map $M_f : \mathcal{X} \to \mathbb{R}^K$, which returns the vector of class logits. The predicted label is given by
\begin{align}\label{eq:f_def}
f(x) := \arg\max_{k \in [K]} M_f(x)_k.
\end{align}
For a labeled example $(x,y)$, define the multiclass margin that measures how much the logit of the true class exceeds the largest competing logit as
$\gamma_f(x, y) := M_f(x)_y - \max_{k \neq y} M_f(x)_k.$
To bound errors under distribution shift, we employ the ramp loss defined as $\ell_r((x, y); f) := r(\gamma_f(x, y))$, 
where $r(t) := \min\{\max(1 - t, 0), 1\}$ is the ramp function which clips the surrogate loss to the interval $[0,1]$. The population ramp loss under distribution $P_{XY}$ is 
$L_r(f, P) := \mathbb{E}_{P_{XY}} [\ell_r((X,Y); f)].$
We will also use the hinge loss $\ell_h((x,y);f) := \max\{1-\gamma_f(x,y),\,0\}$, with population hinge loss under $P_{XY}$ given by
$L_h(f,P) := \mathbb{E}_{(X,Y)\sim P_{XY}}\!\big[\ell_h((X,Y);f)\big]$.

\begin{assumption}
For all $x, x' \in \mathcal{X}$ and $y\in[K]$, $\gamma_f(\cdot, \cdot)$ satisfies 
$
|\gamma_f(x, y) - \gamma_f(x', y)| \leq L_\gamma \|x - x'\|_2.
$
\label{assump:boundlip}
\end{assumption}

In our setting, $f$ is a pre-trained classifier with known \rev{or estimated $L_r(f,P)$ on held-out source validation data.}
Once $f$ is fixed, $L_\gamma$ can be upper-bounded using spectral norm bounds, or \rev{estimated via a data-dependent local gradient-norm bound around observed source samples.}
Only unlabeled target inputs $X \sim Q_X$ are available. Therefore, form deterministic pseudo-labels $\tilde Y := f(X)$, inducing a pseudo-labeled joint distribution $\tilde Q_{XY}$ and an associated score distribution $s\#\tilde Q_{XY}$. Throughout this paper, we use the nonconformity score $s(x,y) := -\gamma_f(x,y)$. 
By definition of $f$ in~\eqref{eq:f_def}, for any $x \in \mathcal X$ and $y \neq f(x)$, we have $\gamma_f(x,f(x)) \ge 0$ and $\gamma_f(x,y) \le 0.$ Hence, 
$s(x,f(x)) \le  s(x,y)$\AQrev{, so} for $(X,Y)\sim Q_{XY}$, we have $s(X,f(X)) \le s(X,Y)$ almost surely. Consequently,
\begin{equation}
\label{eq:stochdominance}
F_{s\#\tilde Q}(t) \ge F_{s\#Q}(t),
\qquad \forall t\in\mathbb{R}.
\end{equation}
\rev{Intuitively, pseudo-calibration trusts the classifier's predictions as labels. For the score $s(x,y)=-\gamma_f(x,y)$ used in this paper, the predicted label has the smallest score, it uses smaller thresholds than oracle calibration, and hence, it can lead to undercoverage under distribution shift.} \AQrev{As hard pseudo-scores are nonpositive, hard pseudo-calibrated sets with finite thresholds are empty or singletons (barring logit ties), so their coverage cannot exceed the classifier accuracy.}
We also make the following assumption.
\begin{assumption}
\label{assump:shift}
The distributions $P_{XY}$ and $Q_{XY}$ satisfy (i) identical label marginals $P_Y = Q_Y$, and (ii) bounded conditional shift: For some $\rho > 0$, we have $\sup_{y \in \mathcal{Y}} W_\infty(P_{X|y}, Q_{X|y}) < \rho.$ 
\end{assumption}
\rev{Assumption~\ref{assump:shift}(i) is standard in domain adaptation analyses to isolate label-conditional covariate shift (e.g.~\cite{kumar2020understanding}).}
Assumption~\ref{assump:shift}(ii) is natural in sensing/control pipelines where perturbations are physically constrained. In such cases, $\rho$ is treated as an a priori parameter.
Under Assumption~\ref{assump:shift} we also have
$W_1(P_{X\mid y}, Q_{X\mid y}) \le \rho$ for all $y$, and hence
$W_1(P_X, Q_X) \le \rho_{\mathrm{mix}} := \sum_{y=1}^K P_Y(y)\,W_1(P_{X\mid y}, Q_{X\mid y}) \le \rho$.

\vspace{-0.4em}

\section{Analytical Results}
We now present upper bounds on the coverage gap under distribution shift. All proofs are in the appendix. 
\rev{Note that $\rho$ and $L_\gamma$ appear only in the following bounds and are not required by the proposed procedures.}
By Assumption~\ref{assump:boundlip}, the score function $s(x,y)$ is $L_\gamma$-Lipschitz in $x$ for every $y \in [K]$. For a given level $\alpha \in (0,1)$, let $q_{P,\alpha}$ denote the \rev{empirical split-conformal threshold computed from $s\#\hat{P}_n$ in}~\eqref{eq:empquantile}. Using this threshold, the achieved coverage with the target distribution $Q$ is $F_{s\#Q}(q_{P,\alpha})$, while the coverage with $P$ is $F_{s\#P}(q_{P,\alpha})$\AQrev{, at least $1-\alpha$ on average}. 
Define the pointwise coverage gap
$\Delta_{P,Q}(\alpha)
:= \big|F_{s\#P}(q_{P,\alpha}) - F_{s\#Q}(q_{P,\alpha})\big|$.
Following \cite{correia2025nonexchangeableconformalpredictionoptimal}, aggregating the discrepancies across $\alpha$ via
$\Delta_{P,Q} := \int_0^1 \Delta_{P,Q}(\alpha) \, d\alpha$ measures the average coverage mismatch when calibrating on $P$ but deploying on $Q$. 

\vspace{-0.4em}

\subsection{Coverage Gap Upper Bounds under Distribution Shift}
\begin{lemma}
\label{lemma:W1upperbounds}
Under Assumptions~\ref{assump:boundlip} and~\ref{assump:shift}, we have
$W_1(s\#P,s\#Q) \le L_\gamma \rho$.
\end{lemma}
For source calibration using labeled data from $P_{XY}$, \AQrev{let $C_s:=\sup_{t \in \mathbb{R}} p_{s\#P}(t)$,} where $p_{s\#P}$ \AQrev{is} the PDF of $s\#P$ (\AQrev{$C_s:=+\infty$ if none exists}). \AQrev{Since the coupling given by Assumption~\ref{assump:shift} moves each score by at most $L_\gamma\rho$, for every calibration sample,
\begin{align}
\Delta_{P,Q}(\alpha),\ \Delta_{P,Q}
\le
\min\{1,\,C_s L_\gamma \rho\},
\label{eq:covgapPQ}
\end{align}}which quantifies the worst-case coverage degradation from source calibration.
\AQrev{For pseudo-calibration, we have the following result.}
\begin{theorem}
\label{thm:pseudo-calibration}
Let Assumptions~\ref{assump:boundlip} and~\ref{assump:shift} hold.
\rev{Let $X_1,\ldots,X_n\stackrel{\rm i.i.d.}{\sim}Q_X$ and let $(X_{n+1},Y_{n+1})\sim Q_{XY}$ be independent. Let $\mathbb P_Q$ denote the joint law
of $X_{1:n}$ and $(X_{n+1},Y_{n+1})$.}
Define
$\tilde Y_i := f(X_i)$, so that $(X_i,\tilde Y_i)\stackrel{\text{i.i.d.}}{\sim}\tilde Q_{XY}$.
Let $q_{\tilde Q,\alpha}$ be the split-conformal threshold computed from $\{s(X_i,\tilde Y_i)\}_{i=1}^n$, and define the pseudo-calibrated set
\begin{align}
C^{1-\alpha}_{\tilde Q}(X_{n+1})
:=
\big\{ y \in [K] : s(X_{n+1},y) \le q_{\tilde Q,\alpha} \big\}.\nonumber
\end{align}
Then the marginal coverage on the target domain satisfies
{\small
\begin{equation}
\label{eq:pseudo_cov_lb_ramp}
\rev{\mathbb P_Q}\!\big( Y_{n+1} \in C_{\tilde Q}^{1-\alpha}(X_{n+1}) \big)
\ge
1-\alpha - L_r(f,P) - L_\gamma\rho_{\mathrm{mix}},
\end{equation}}\rev{with the right-hand side clipped at 0 when it becomes negative. 
Moreover, $\mathbb P_Q\!\big(Y_{n+1}\in C^{1-\alpha}_{\tilde Q}(X_{n+1})\big)\ge 1-\alpha-L_r(f,Q)$, where $L_r(f,Q)\le L_r(f,P)+L_\gamma\rho_{\mathrm{mix}}$.} \AQrev{If the pseudo-scores are continuous and $\lceil(n+1)(1-\alpha)\rceil\le n$, the expected size of $C^{1-\alpha}_{\tilde Q}(X_{n+1})$ is $\lceil(n+1)(1-\alpha)\rceil/(n+1)$ under any shift.}
\end{theorem}

\rev{Note that when $P_Y \neq Q_Y$, the result can be relaxed with an additional total variation distance term between the marginals.}
This result shows that the hard pseudo-calibration method gives \AQrev{lower} coverage guarantees controlled by the source ramp loss and the shift magnitude\AQrev{, not nominal target coverage}.
\rev{We use hard pseudo-calibration as a reference construction because it is the simplest unlabeled-target baseline directly analyzed by Theorem~\ref{thm:pseudo-calibration}. \AQrev{Corollary~\ref{cor:hinge_tau} and Theorem~\ref{thm:relative_coverage} extend the analysis to the inflated and source-tuned sets.}}

\begin{corollary}
\label{cor:hinge_tau}
Under the setup of Theorem~\ref{thm:pseudo-calibration}, for any fixed $\tau\ge0$, define the relaxed prediction set $C^{1-\alpha}_{\tilde Q, \tau}(X_{n+1})
:= \big\{ y \in [K] : s(X_{n+1},y) \le q_{\tilde Q,\alpha} + \tau \big\}$.
Then the marginal coverage on the target domain satisfies
{\small \begin{equation}
\label{eq:pseudo_cov_lb_hinge}
\rev{\mathbb P_Q}\!\big( Y_{n+1} \in C^{1-\alpha}_{\tilde Q,\tau}(X_{n+1}) \big)
\ge
1-\alpha - \min \left\{L_r(f,Q), \frac{L_h(f,Q)}{1+\tau/2}\right\}.
\end{equation}}
\end{corollary}

\begin{table*}[t]
\centering
\begingroup
\setlength{\abovecaptionskip}{4pt}
\setlength{\belowcaptionskip}{2pt}
\setlength{\tabcolsep}{2.0pt}
\renewcommand{\arraystretch}{0.8}
\scriptsize
\caption{Coverage and ESS at shift levels $\sigma$ with $1-\alpha=0.8$, averaged over five independent runs. Boldface marks the least conservative non-oracle method meeting the nominal coverage target, or the closest method to $1-\alpha$ if none meets it. \AQrev{ECP and WQLCP are evaluated on five classifiers retrained with the same pipeline.}}

\label{tab:all_cov_ess}

\resizebox{0.9\textwidth}{!}{%
\begin{tabular}{@{}c|c|cc|cc|cc|cc|cc|cc@{}}
\hline
Dataset & $\sigma$
& \multicolumn{2}{c|}{Source Cal}
& \multicolumn{2}{c|}{Hard Pseudo Cal}
& \multicolumn{2}{c|}{Source-tuned Cal}
& \multicolumn{2}{c|}{ECP}
& \multicolumn{2}{c|}{WQLCP}
& \multicolumn{2}{c}{Target Cal} \\
& 
& Cov (\%) & ESS
& Cov (\%) & ESS
& Cov (\%) & ESS
& Cov (\%) & ESS
& Cov (\%) & ESS
& Cov (\%) & ESS \\
\hline
\multirow{3}{*}{MNIST}
& 0.7
& 72.12 & 0.73
& 78.36 & 0.80
& \textbf{90.14} & \textbf{0.97}
& \AQrev{98.85} & \AQrev{2.01}
& \AQrev{99.97} & \AQrev{8.23}
& 79.76 & 0.81 \\

& 1.6
& 37.21 & 0.45
& 51.34 & 0.80
& 75.72 & 2.65
& \textbf{\AQrev{80.37}} & \textbf{\AQrev{3.21}}
& \AQrev{99.45} & \AQrev{8.82}
& 80.39 & 3.18 \\

& 2.0
& 23.92 & 0.41
& 33.41 & 0.80
& 52.54 & 2.27
& \AQrev{61.66} & \AQrev{3.11}
& \textbf{\AQrev{96.87}} & \textbf{\AQrev{8.90}}
& 80.28 & 5.52 \\
\hline

\multirow{3}{*}{CIFAR-10}
& 0.3
& 75.93 & 0.91
& 70.02 & 0.79
& \textbf{90.08} & \textbf{1.46}
& \AQrev{77.20} & \AQrev{0.95}
& \AQrev{99.07} & \AQrev{3.52}
& 80.18 & 1.02 \\

& 0.9
& 51.53 & 0.86
& 49.46 & 0.80
& \textbf{90.88} & \textbf{3.84}
& \AQrev{60.51} & \AQrev{1.18}
& \AQrev{95.14} & \AQrev{5.02}
& 79.83 & 2.34 \\

& 1.5
& 29.95 & 0.83
& 29.11 & 0.79
& 78.42 & 4.84
& \AQrev{38.58} & \AQrev{1.24}
& \textbf{\AQrev{84.61}} & \textbf{\AQrev{5.75}}
& 80.02 & 5.06 \\
\hline

\multirow{3}{*}{CIFAR-100}
& 0.3
& 75.36 & 4.52
& 41.98 & 0.80
& \textbf{93.07} & \textbf{15.23}
& \AQrev{94.81} & \AQrev{17.30}
& \AQrev{98.96} & \AQrev{37.88}
& 79.57 & 5.76 \\

& 0.5
& 66.63 & 4.95
& 33.23 & 0.80
& 91.54 & 22.08
& \textbf{\AQrev{90.76}} & \textbf{\AQrev{19.80}}
& \AQrev{97.90} & \AQrev{43.26}
& 80.49 & 10.34 \\

& 0.7
& 55.95 & 5.01
& 26.35 & 0.81
& 86.14 & 24.56
& \textbf{\AQrev{84.19}} & \textbf{\AQrev{20.47}}
& \AQrev{95.31} & \AQrev{46.12}
& 79.66 & 17.05 \\
\hline
\end{tabular}%
}
\endgroup
\end{table*}

\vspace{-1.0em}

\subsection{Source-Tuned Pseudo-Calibration}
\label{subsec:pac_labels}

From~\eqref{eq:stochdominance}, pseudo-labeling is pessimistic in score space \AQrev{and tends to undercover}.
\rev{To mitigate this pessimism, we keep the pseudo-label where the classifier is confident and randomize where it is uncertain. Since random labels can carry larger scores, this lifts the threshold and enlarges sets where the pseudo-label is least reliable.}

Given a function $H:\mathcal X\to \mathbb{R}_{\ge 0}$ that measures some notion of uncertainty (e.g. predictive entropy) \AQrev{and} a threshold $u \in \mathcal U$, define, for $x\in\mathcal X$, the quantity
\begin{align}\label{eq:randompseudolabel} 
    \tilde Y_u(x) = f(x) \mathbbm{1}\{H(x)\leq u\} + U \mathbbm{1}\{H(x)> u\},
\end{align}
where $U \sim \text{Unif}([K])$ and $\mathbbm{1}\{\cdot\}$ is the indicator function. Let $\tilde P^u_{XY}$ and $\tilde Q^u_{XY}$ be the randomized pseudo-labeled source and target distributions induced by $\tilde Y_u$, with score distributions $s\#\tilde P^u$ and $s\#\tilde Q^u$. Since $f(x)$ minimizes $s(x,y)$ over $y\in[K]$, randomization can only increase the scores\AQrev{, i.e., $s(x,\tilde Y_u(x)) \ge s(x,f(x))$}. Consequently, at fixed $\alpha$, the threshold and the prediction sets computed from \rev{$\tilde Q^u_{XY}$} are never smaller than under \rev{$\tilde{Q}_{XY}$}.

We tune $u$ on source data. For each $u$, we compute the threshold using the mixed pseudo-labeled scores\AQrev{, evaluate its coverage on a separate labeled source tuning set,} and then select $u^\star$ such that the empirical source coverage stays above $1-\alpha$ \AQrev{when a feasible grid value exists}. With $u^\star$ fixed, we pseudo-label the target samples and compute the final threshold \AQrev{(Algorithm~\ref{alg:source-tuned-pseudo})}.

\begin{algorithm}[t]
\caption{Source-Tuned Pseudo-Calibration (STPC)}
\label{alg:source-tuned-pseudo}
\begin{algorithmic}[1]\small
\rev{
\REQUIRE Source \AQrev{inputs $\{X_i^P\}_{i=1}^m$, labeled source tuning set $\mathcal D_{\rm tun}$}, target inputs $\{X_j^Q\}_{j=1}^n$, classifier $f$, uncertainty $H$, score $s$, level $\alpha$, \AQrev{finite, nonempty} grid $\mathcal U$.
\STATE For each $u\in\mathcal U$, form $\tilde Y_u$ by~\eqref{eq:randompseudolabel}, compute source pseudo-scores $S_i^u=s(X_i^P,\tilde Y_u(X_i^P))$, and let $q_{\tilde P^u,\alpha}$ be their split-conformal quantile.
\STATE Compute \AQrev{the coverage $\hat c(u)$ of $q_{\tilde P^u,\alpha}$ on $\mathcal D_{\rm tun}$} and set
$u^\star=\max\{u\in\mathcal U:\hat c(u)\ge 1-\alpha\}$ \AQrev{if this set is nonempty, and $u^\star=\min\mathcal U$ otherwise}.
\STATE Compute target pseudo-scores $S_j^{Q,u^\star}=s(X_j^Q,\tilde Y_{u^\star}(X_j^Q))$ and return their split-conformal quantile $q_{\tilde Q^{u^\star},\alpha}$.}
\end{algorithmic}
\end{algorithm}

\rev{Let $\nu_K$ denote the uniform distribution on $[K]$.
For fixed $u$, let $q_{\tilde Q^u,\alpha}$ be the split-conformal threshold computed from $\mathcal D_Q^{\rm cal}=\{X_i\}_{i=1}^n$ and the auxiliary random labels $U_{1:n}$.
Define
$r_u(t):=(Q_X\otimes \nu_K)
\{(x,v):H(x)>u,\ s(x,f(x)) \le t<s(x,v)\}$,
and
$R_u:=\mathbb E_{\mathcal D_Q^{\rm cal}, U_{1:n}}
\left[r_u(q_{\tilde Q^u,\alpha})\right] \ge 0.$
Let $\mathbb P_{Q,U}$ denote the joint law of $\mathcal D_Q^{\rm cal}$, $U_{1:n}$, and $(X,Y)\sim Q_{XY}$.}

\begin{theorem}
\label{thm:relative_coverage}
\rev{Let Assumptions~\ref{assump:boundlip} and~\ref{assump:shift} hold, fix $u$. Then
\begin{enumerate}[$(i)$]
\item $\mathbb P_{Q,U}\big(Y\in C^{1-\alpha}_{\tilde Q^u}(X)\big) \ge \mathbb P_{Q,U}\big(Y\in C^{1-\alpha}_{\tilde Q}(X)\big)$.
\item $\mathbb P_{Q,U}\big(Y\in C^{1-\alpha}_{\tilde Q^u}(X)\big) \ge 1-\alpha-L_r(f,Q)+R_u .$
\end{enumerate}
}
\end{theorem}
\rev{Part~$(i)$ guarantees that source-tuned coverage is never below hard pseudo-calibration. Part~$(ii)$ improves the hard pseudo-calibration lower-bound floor.
For the data-chosen value $u^\star$ in Algorithm~\ref{alg:source-tuned-pseudo}, both conclusions of Theorem~\ref{thm:relative_coverage} continue to hold because $u^\star$ is selected using an independent source-tuning set\AQrev{, with $R_u$ replaced by $\mathbb E[R_{u^\star}]$}.}
\AQrev{In the proof, target coverage equals $\bar F_u+R_u-M_u$, where $\bar F_u\ge 1-\alpha$ and $M_u\le L_r(f,Q)$ are the mixed pseudo-label coverage and the missed mass, averaged over calibration samples. Thus, nominal coverage holds if and only if $R_u-M_u\ge 1-\alpha-\bar F_u$. As the right-hand side is nonpositive, $R_u\ge M_u$ suffices, e.g., $R_u\ge L_r(f,Q)$, which source tuning alone does not ensure.}

\vspace{-0.5em}

\section{Numerical Experiments}

We evaluate MNIST, CIFAR-10, and CIFAR-100. The source distribution $P$ is the original dataset, and $Q_\sigma$ is obtained by applying a stochastic image transform consisting of an appearance change and Gaussian noise of strength $\sigma$. We train $f$ on source data, use split CP with $\alpha=0.2$, and take predictive entropy as $H$. MNIST uses an autoencoder-based MLP trained for 30 epochs using SGD with learning rate (LR) $10^{-4}$. CIFAR-10 \AQrev{and CIFAR-100 use convolutional autoencoder classifiers} trained for 80 epochs using Adam with LR $10^{-3}$. All use a batch size $128$ and cross-entropy/hinge\AQrev{/reconstruction} losses.

We compare source calibration, hard pseudo-calibration, source-tuned pseudo-calibration, and the unlabeled-target baselines ECP and WQLCP~\cite{kasa2025adaptingpredictionsetsdistribution,Alijani_2025_CVPR}. \AQrev{ECP is calibrated on labeled source data with its default squared entropy scaling. WQLCP follows~\cite{Alijani_2025_CVPR} with our convolutional VAE.}
\rev{We report empirical coverage and expected set size (ESS), where ESS is the average test-set prediction-set cardinality, $\frac{1}{|\mathcal D^{Q_\sigma}_{\mathrm{tst}}|}\sum_{x\in\mathcal D^{Q_\sigma}_{\mathrm{tst}}}|C(x)|$.}
All curves are averaged over five independent runs with fresh calibration/test splits and retraining. \rev{Shaded bands denote one standard deviation (std) across runs. The code is available at \url{https://github.com/Farbod-Siahkali/Source-Tuned-Pseudo-Calibration}}.

\begin{figure}[t]
  \centering
  \begin{subfigure}[b]{0.49\linewidth}
    \centering
    \includegraphics[width=\linewidth]{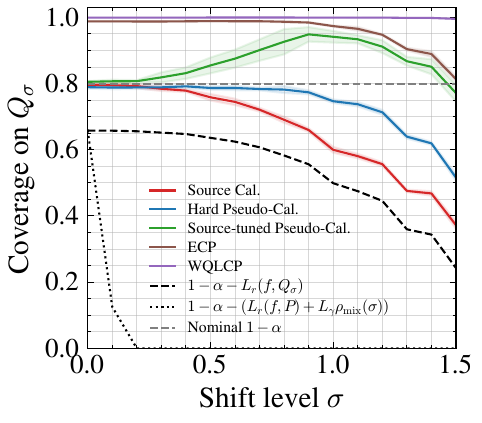}
    \caption{MNIST}
  \end{subfigure}
  \hfill
  \begin{subfigure}[b]{0.49\linewidth}
    \centering
    \includegraphics[width=\linewidth]{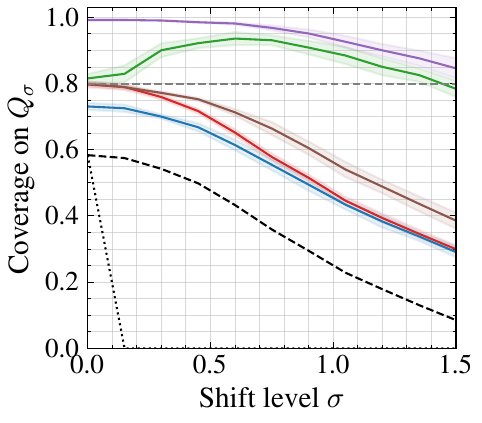}
    \caption{CIFAR-10}
  \end{subfigure}
  \caption{Coverage on $Q_\sigma$ vs. $\sigma$ for various methods. Dashed curves show the coverage bounds from Theorem~\ref{thm:pseudo-calibration}.}
  \label{fig:cov_vs_sigma_all}
\end{figure}

\begin{figure}[t]
  \centering
  \begin{subfigure}[b]{0.49\linewidth}
    \centering
    \includegraphics[width=\linewidth]{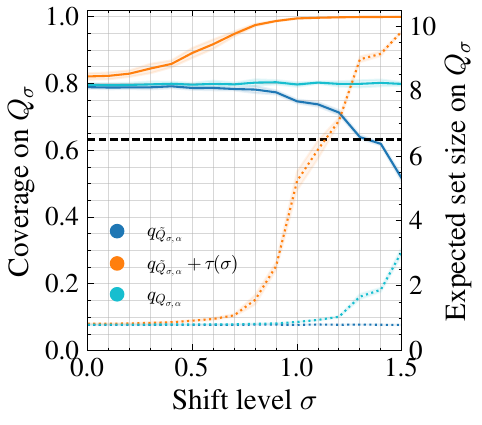}
    \caption{MNIST}
  \end{subfigure}
  \hfill
  \begin{subfigure}[b]{0.49\linewidth}
    \centering
    \includegraphics[width=\linewidth]{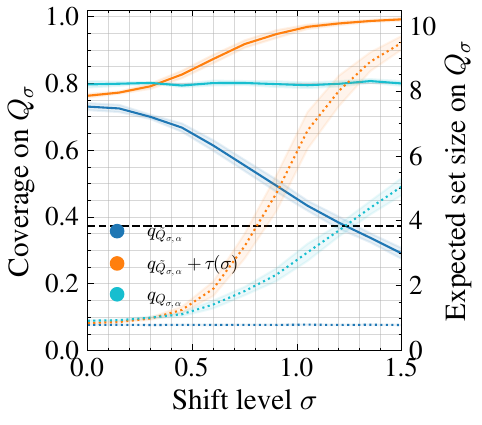}
    \caption{CIFAR-10}
  \end{subfigure}
  \caption{Coverage (solid) and ESS (dashed) on $Q_\sigma$: hard pseudo-calibration \rev{$q_{\tilde Q_\sigma,\alpha}$}, $\tau$-adjusted \rev{$q_{\tilde Q_\sigma,\alpha}+\tau(\sigma)$}, and oracle $q_{Q_\sigma,\alpha}$. The black curve is the hinge-loss lower bound.}
\label{fig:tau_adjusted}
\end{figure}

\rev{Table~\ref{tab:all_cov_ess} summarizes coverage and ESS at representative shifts. ESS can be below one because, for some test inputs, every class score may exceed the calibrated threshold, resulting in an empty prediction set.} Source and hard pseudo-calibration degrade as $\sigma$ increases, while source-tuned pseudo-calibration improves coverage, often approaching oracle at moderate shifts.
\rev{This improvement can be conservative because randomization raises the pseudo-scores and the threshold, potentially producing larger sets than necessary to attain nominal coverage. This effect is most visible on CIFAR-100, where the more difficult 100-class task yields more uncertain predictions and weaker classification margins under shift.}
\AQrev{At $\sigma=0.3$, for example, source-tuned pseudo-calibration attains $93.07\%$ coverage with ESS $15.23$, versus $79.57\%$ and $5.76$ for the oracle. Both use the same score, and their ESS nearly coincide at matched coverage.}

\rev{To connect these observations with our theoretical results, we evaluate the lower bound implied by Theorem~\ref{thm:pseudo-calibration}. For MNIST and CIFAR-10, we estimate $L_r(f, Q_\sigma)$ using the oracle labeled samples from $Q_\sigma$, and compute the corresponding lower bound with $\tau = 0$. Fig.~\ref{fig:cov_vs_sigma_all} plots empirical coverage \AQrev{of all methods} on $Q_\sigma$ with the theoretical bounds. Although conservative, the bounds track the coverage degradation across the full range of $\sigma$, and remain below the empirical coverage of hard pseudo-calibration.} \AQrev{The label-free curve, with plug-in $\rho_{\mathrm{mix}}$ and $L_\gamma$, is a worst-case bound over all shifts with the same $\rho_{\mathrm{mix}}$. The noise in $Q_\sigma$ is unbounded, but the proof of Theorem~\ref{thm:pseudo-calibration} needs only a finite $\rho_{\mathrm{mix}}$.}

\AQrev{Corollary~\ref{cor:hinge_tau} covers any $\tau$, including a $\sigma$-dependent one, fixed independently of the target calibration sample and test point. Our diagnostic instead estimates $L_h(f,Q_\sigma)$ in~\eqref{eq:pseudo_cov_lb_hinge} from labeled target data, so Corollary~\ref{cor:hinge_tau} does not directly apply. It only illustrates threshold inflation.}

\rev{Let $\Delta_P := (1-\alpha) - \mathbb P_P\!\left(Y\in C^{1-\alpha}_{P,0}(X)\right)$ where $C^{1-\alpha}_{P,0}$ is obtained by calibrating with pseudo-labels. We estimate $L_h(f,P)$ on labeled source data and, for this illustrative experiment only, assume $L_h(f,Q_\sigma)$ is known. Choosing $\tau$ so that the hinge-loss term $L_h(f,Q_\sigma)/(1+\tau/2)$ equals $L_h(f,P)-\Delta_P$ gives
$\tau(\sigma)
=
2\left(\frac{L_h(f,Q_\sigma)}{L_h(f,P)-\Delta_P}-1\right),
$
and we build prediction sets on $Q_\sigma$ using the relaxed threshold $q_{\tilde Q_\sigma,\alpha}+\tau(\sigma)$.} \AQrev{We set $\tau(\sigma)=0$ if it is negative or if $L_h(f,P)-\Delta_P$ is numerically zero or negative.} As shown in Fig.~\ref{fig:tau_adjusted}, the unadjusted pseudo-calibration drops as the shift increases, whereas the adjusted scheme \rev{improves empirical} coverage at the cost of larger ESS.

\vspace{-0.5em}
\section{Conclusion}
We studied conformal prediction under distribution shift with unlabeled target data. We derived coverage lower bounds that explicitly connect target-domain coverage to classifier properties and distribution shift measures. Building on these results, we proposed a source-tuned pseudo-calibration method that interpolates between hard pseudo-labels and randomized labels using an uncertainty measure, mitigating the pessimism inherent in pseudo-calibration. \rev{Experiments show improved coverage under shift, at the cost of larger prediction sets.}

\bibliographystyle{IEEEtran}
\bibliography{references}

\newpage

\appendix

\setcounter{lemma}{1}

\subsection{Proof of Lemma~\ref{lemma:W1upperbounds}}

Fix $\epsilon>0$. For each $y$, pick $\pi_y\in\Pi(P_{X\mid y},Q_{X\mid y})$ with
\begin{equation}\label{eq:cond_coupling}
\operatorname*{ess\,sup}_{(X,X')\sim\pi_y}\|X-X'\|_2\le \rho+\epsilon.
\end{equation}
Let $Y\sim P_Y$, and conditional on $Y=y$ sample $(X,X')\sim\pi_y$. Then
$(X,Y)\sim P_{XY}$ and $(X',Y)\sim Q_{XY}$, and
$|s(X,Y)-s(X',Y)|\le L_\gamma\|X-X'\|_2$ a.s. Hence
$
W_1(s\#P,s\#Q)
\le L_\gamma\,\mathbb E\|X-X'\|_2 \le L_\gamma(\rho+\epsilon)$.
Let $\epsilon\rightarrow 0$.

\vspace{-10pt}

\AQrev{\subsection{Proof of~\eqref{eq:covgapPQ}}\label{app:proof_eq5}
Since $\sup_y W_\infty(P_{X|y},Q_{X|y})<\rho$ (Assumption~\ref{assump:shift}), there is a coupling with $X\sim P_{X|Y}$, $X'\sim Q_{X|Y}$ and $\|X-X'\|_2\le\rho$ a.s., so by Assumption~\ref{assump:boundlip}, $S:=s(X,Y)$ and $S':=s(X',Y)$ satisfy $|S-S'|\le d:=L_\gamma\rho$ a.s. Hence $F_{s\#P}(t-d)\le F_{s\#Q}(t)\le F_{s\#P}(t+d)$, so $|F_{s\#P}(t)-F_{s\#Q}(t)|\le\min\{1,C_sd\}$ for every $t$, in particular at $t=q_{P,\alpha}$ for every calibration sample. Integrating over $\alpha$ bounds $\Delta_{P,Q}$.}

\subsection{Kantorovich-Rubinstein Inequality}
\begin{lemma}
\label{lemma:KR}
Let $(\mathcal{X},d)$ be a metric space, let $P,Q$ be probability measures on $\mathcal{X}$, and let $f:\mathcal{X}\to\mathbb{R}$ be $L$-Lipschitz \AQrev{and integrable under $P$ and $Q$}. Then
$
\big|\mathbb{E}_P[f(X)] - \mathbb{E}_Q[f(X')]\big|
\le L\,W_1(P,Q)$.
\end{lemma}
\begin{IEEEproof}
For any coupling $\pi\in\Pi(P,Q)$ with $(X,X')\sim\pi$,
$|\mathbb E[f(X)-f(X')]| \le \mathbb E|f(X)-f(X')|
\le L\,\mathbb E[d(X,X')]$.
Taking $\inf_{\pi\in\Pi(P,Q)}$ yields the claim.
\end{IEEEproof}

\vspace{-10pt}

\subsection{Proof of Theorem~\ref{thm:pseudo-calibration}}
Let $\mathbb P_Q$ be as in Theorem~\ref{thm:pseudo-calibration}. Define
$S_{n+1}:=s(X_{n+1},Y_{n+1})$, $\tilde S_{n+1}:=s(X_{n+1},f(X_{n+1}))$, and
$\tilde S_i:=s(X_i,f(X_i))$ for $i\le n$. Let $(X,Y)\sim Q_{XY}$ and set
$S:=s(X,Y)$, $\tilde S:=s(X,f(X))$. By split conformal validity,
$\mathbb P_Q(\tilde S_{n+1}\le q_{\tilde Q,\alpha})\ge 1-\alpha$. \AQrev{As $s(X_{n+1},y)>0$ for $y\neq f(X_{n+1})$, a finite (negative) $q_{\tilde Q,\alpha}$ gives $|C^{1-\alpha}_{\tilde Q}(X_{n+1})|=\mathbbm 1\{\tilde S_{n+1}\le q_{\tilde Q,\alpha}\}$, so the expected set size is this probability. Since $\tilde S_1,\ldots,\tilde S_{n+1}$ do not involve labels and are i.i.d., by exchangeability it equals $\lceil(n+1)(1-\alpha)\rceil/(n+1)$ when they have no ties.}
Hence, for any $\tau\ge0$,
{\small \begin{equation}
\mathbb P_Q\big( S_{n+1} > q_{\tilde Q,\alpha} + \tau \big)
\le
\alpha + Q\big( S - \tilde S > \tau \big).
\label{eq:cor_union_to_margin}
\end{equation}}

If $Y = f(X)$, then $S = \tilde S$.
Thus 
$\{S - \tilde S > \tau\} \subseteq \{Y \neq f(X)\}.$
On the event $\{Y \neq f(X)\}$,
$
S - \tilde S = -\gamma_f(X,Y) + \gamma_f(X,f(X)).
$
We have
$
0 \le \gamma_f(X,f(X)) \le -\gamma_f(X,Y),
$
and therefore
$S - \tilde S
\le 2\bigl(-\gamma_f(X,Y)\bigr)$.
Hence,
$\{S - \tilde S > \tau\}
\subseteq
\Bigl\{\, 2(-\gamma_f(X,Y)) > \tau \,\Bigr\}
=
\Bigl\{\, \gamma_f(X,Y) < -\tfrac{\tau}{2} \,\Bigr\}$, and consequently,
\begin{equation}
\label{eq:cor_start}
Q\bigl(S - \tilde S > \tau\bigr)
\le
Q\Bigl( \gamma_f(X,Y) \le -\tfrac{\tau}{2} \Bigr).
\end{equation}
Next we control $Q(\gamma_f(X,Y) \le -\tfrac{\tau}{2})$ via the ramp loss.
Recall the ramp loss $\ell_r((x,y);f) = r(\gamma_f(x,y))$.
For any $\gamma \le 0$ we have $1-\gamma \ge 1$, hence
$
r(\gamma) = \min\{(1-\gamma)^+,1\} = 1.
$
Thus, on the event $\{\gamma_f(X,Y) \le -\tfrac{\tau}{2}\}$ we have
$
\mathbbm{1}\{\gamma_f(X,Y) \le -\tfrac{\tau}{2}\} \le
r\bigl(\gamma_f(X,Y)\bigr)$.
Taking expectations under $Q_{XY}$ yields
\[
Q\!\left(\gamma_f(X,Y) \le -\tfrac{\tau}{2}\right) 
\le \mathbb{E}_{Q}\!\left[r\bigl(\gamma_f(X,Y)\bigr)\right] = L_r(f,Q).
\]
Since $\tau\ge0$, we have
$Q\!\left(\gamma_f(X,Y) \le -\tfrac{\tau}{2}\right) \le
Q\!\left(\gamma_f(X,Y) \le 0\right) \le
L_r(f,Q)$. 

For each $y\in[K]$, define $\ell_y(x) := r\big(\gamma_f(x,y)\big)$.
By Assumption~\ref{assump:boundlip}, and since $r(\cdot)$ is 1-Lipschitz, $\ell_y$ is $L_\gamma$-Lipschitz in $x$.
We can write
$
L_r(f,P)
=
\sum_{y=1}^K P_Y(y) \, \mathbb{E}_{P_{X\mid y}}[\ell_y(X)],
$
and $L_r(f,Q)
=
\sum_{y=1}^K Q_Y(y) \, \mathbb{E}_{Q_{X\mid y}}[\ell_y(X)]$.
By Assumption~\ref{assump:shift}(i) \AQrev{and Lemma~\ref{lemma:KR} applied to each $\ell_y$,}
{\small\begin{align*}
&\big|L_r(f,Q) - L_r(f,P)\big|
\AQrev{=\Big|\sum_{y} P_Y(y)\big(\mathbb{E}_{Q_{X\mid y}}-\mathbb{E}_{P_{X\mid y}}\big)[\ell_y]\Big|}\\
&\quad\le L_\gamma \sum_{y} P_Y(y) W_1(P_{X\mid y},Q_{X\mid y})=L_\gamma \rho_{\mathrm{mix}},
\end{align*}}and $L_r(f,Q) \le L_r(f,P) + L_\gamma \rho_{\mathrm{mix}}$.
Combining yields
$
Q\big( S - \tilde S > 0 \big)
\le
{L_r(f,P) + L_\gamma \rho_{\mathrm{mix}}}.
$
Thus
$\mathbb P_Q\big( S_{n+1} > q_{\tilde Q,\alpha} \big) \le \alpha +
{L_r(f,P) + L_\gamma \rho_{\mathrm{mix}}}$.
Finally,
$\mathbb P_Q\big( Y_{n+1} \in C^{1-\alpha}_{\tilde Q}(X_{n+1}) \big)
=
\mathbb P_Q\big( S_{n+1} \le q_{\tilde Q,\alpha} \big),
$
which gives \eqref{eq:pseudo_cov_lb_ramp}.

\vspace{-10pt}

\subsection{Proof of Corollary~\ref{cor:hinge_tau}}
Fix $\tau\ge 0$. Let $(X,Y) \sim Q_{XY}$, and define $\tilde Y := f(X)$,
$S := s(X,Y)$, $\tilde S := s(X,\tilde Y)$.
For the relaxed set
$C^{1-\alpha}_{\tilde Q,\tau}(X_{n+1})=\{y: s(X_{n+1},y)\le q_{\tilde Q,\alpha}+\tau\}$,
we have
$
\{Y_{n+1}\notin C^{1-\alpha}_{\tilde Q,\tau}(X_{n+1})\}
=
\{S_{n+1}>q_{\tilde Q,\alpha}+\tau\}$.
Using~\eqref{eq:cor_union_to_margin} and~\eqref{eq:cor_start},
$\mathbb P_Q(S_{n+1}>q_{\tilde Q,\alpha}+\tau)
\le
\alpha + Q(S-\tilde S>\tau)\le
\alpha + Q\!\left(\gamma_f(X,Y)\le -\tfrac{\tau}{2}\right)$.
On the event $\{\gamma_f(X,Y)\le -\tfrac{\tau}{2}\}$ we have
$1-\gamma_f(X,Y)\ge 1+\tfrac{\tau}{2}$, hence
$\ell_h((X,Y);f)=\max\{1-\gamma_f(X,Y),0\}\ \ge\ 1+\tfrac{\tau}{2}$.
Therefore,
$
\mathbbm{1}\!\left\{\gamma_f(X,Y)\le -\tfrac{\tau}{2}\right\}
\le
\frac{\ell_h((X,Y);f)}{1+\tfrac{\tau}{2}}$.
Taking expectation under $Q_{XY}$ yields
$
Q\!\left(\gamma_f(X,Y)\le -\tfrac{\tau}{2}\right)
\le
\frac{L_h(f,Q)}{1+\tfrac{\tau}{2}}.
$
Combining gives
$
\mathbb P_Q(S_{n+1}>q_{\tilde Q,\alpha}+\tau) \le
\alpha + \frac{L_h(f,Q)}{1+\tfrac{\tau}{2}}
$
and~\eqref{eq:pseudo_cov_lb_hinge}.

\vspace{-10pt}

\subsection{Proof of Theorem~\ref{thm:relative_coverage}}
\emph{Part (i).} \AQrev{Let $\tilde S_i:=s(X_i,f(X_i))$ and $S_i^u:=s(X_i,\tilde Y_u(X_i))$ for the calibration inputs $\mathcal D_Q^{\rm cal}=\{X_i\}_{i=1}^n$ and randomizations $U_{1:n}$. If $H(X_i)\le u$, then $S_i^u=\tilde S_i$, and otherwise $S_i^u\ge\tilde S_i$, since $f(X_i)$ minimizes $s(X_i,\cdot)$. Thus $\{S_i^u\le t\}\subseteq\{\tilde S_i\le t\}$ for every $t$, so $q_{\tilde Q^u,\alpha}\ge q_{\tilde Q,\alpha}$ for every realization of $(\mathcal D_Q^{\rm cal},U_{1:n})$. Hence $C^{1-\alpha}_{\tilde Q^u}(x)\supseteq C^{1-\alpha}_{\tilde Q}(x)$ for every $x$, so $\mathbbm 1\{Y_{n+1}\in C^{1-\alpha}_{\tilde Q^u}(X_{n+1})\}\ge\mathbbm 1\{Y_{n+1}\in C^{1-\alpha}_{\tilde Q}(X_{n+1})\}$ pointwise, and taking expectations gives (i).}

\rev{\emph{Part (ii).}
For deterministic $t$, define
$F_u(t)=\Pr\{s(X,\tilde Y_u(X))\le t\}$, $F_0(t)=Q_X\{s(X,f(X))\le t\}$, $F_Q(t)=Q_{XY}\{s(X,Y)\le t\}$.
Since $f(x)$ minimizes $s(x,y)$, randomization can only increase the pseudo-score, and $F_0(t)-F_u(t)=r_u(t)$.
Also, $F_0(t)-F_Q(t)=Q_{XY}\{s(X,f(X))\le t<s(X,Y)\}
\le L_r(f,Q)$.
Hence, for every deterministic $t$, \AQrev{$F_Q(t)=F_u(t)+r_u(t)-m(t)$ with $m(t):=F_0(t)-F_Q(t)\le L_r(f,Q)$, so}
$F_Q(t)\ge F_u(t)-L_r(f,Q)+r_u(t)$.
Since $q_{\tilde Q^u,\alpha}$ is a function of the calibration data $(\mathcal D_Q^{\rm cal},U_{1:n})$, the marginal coverage is
$\mathbb P_{Q,U}\big(Y\in C_{\tilde Q^u}^{1-\alpha}(X)\big)
=\mathbb E_{\mathcal D_Q^{\rm cal},U_{1:n}}\!\big[F_Q(q_{\tilde Q^u,\alpha})\big]$.
Evaluating the last inequality at $t=q_{\tilde Q^u,\alpha}$ and taking expectation over $(\mathcal D_Q^{\rm cal},U_{1:n})$, split-conformal validity gives
$\mathbb E[F_u(q_{\tilde Q^u,\alpha})]\ge 1-\alpha$ and $\mathbb E[r_u(q_{\tilde Q^u,\alpha})]=R_u$. Therefore,
\[
\mathbb P_{Q,U}\big(Y\in C_{\tilde Q^u}^{1-\alpha}(X)\big) \ge
1-\alpha-L_r(f,Q)+R_u.
\]
}
\AQrev{As $u^\star$ depends only on independent source data, conditioning on $u^\star$ and averaging gives (i), and (ii) with $\mathbb E[R_{u^\star}]$.}

\clearpage
\setcounter{table}{0}\setcounter{figure}{0}
\renewcommand{\thetable}{S-\Roman{table}}\renewcommand{\thefigure}{S\arabic{figure}}
\begin{figure}[hb]
\centering
\begin{subfigure}[b]{0.48\columnwidth}
\centering
\includegraphics[width=\linewidth]{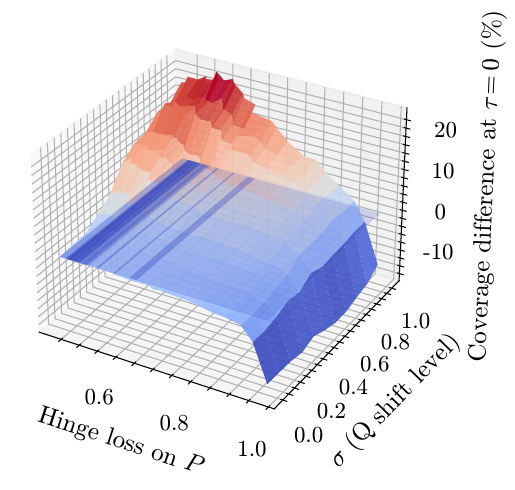}
\caption{MNIST}
\label{fig:covdiff_mnist}
\end{subfigure}
\hfill
\begin{subfigure}[b]{0.48\columnwidth}
\centering
\includegraphics[width=\linewidth]{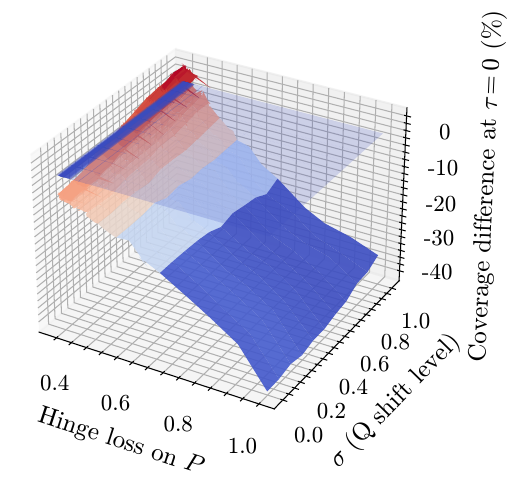}
\caption{CIFAR-10}
\label{fig:covdiff_cifar}
\end{subfigure}
\caption{Coverage difference of hard pseudo-calibration and source calibration. Positive values indicate that pseudo-calibration has smaller coverage loss.}
\label{fig:covdiff_mnist_cifar}
\end{figure}

\subsection{Implementation Details}
\label{app:impl_details}

For each dataset, the source classifier $f$ is trained only on the source-domain training split and then kept fixed during calibration and evaluation. Predictive entropy of the classifier output is used as the uncertainty score $H$ in source-tuned pseudo-calibration.

\paragraph{MNIST}
We use an autoencoder-based classifier. Each $28\times 28$ grayscale image is flattened and encoded by a two-layer MLP
$784 \rightarrow 256 \rightarrow 64$, with ReLU activations, and decoded by $64 \rightarrow 256 \rightarrow 784$.
A spectral-normalized linear head maps the $64$-dimensional latent representation to $10$ logits. Training uses SGD with momentum $0.9$, batch size $128$, $30$ epochs, initial learning rate $10^{-4}$, and weight decay $5\times 10^{-3}$, with a multi-step decay at $50\%$ and $75\%$ of training. The loss is
$\mathcal L = \mathcal L_{\mathrm{CE}} + 4\,\mathcal L_{\mathrm{hinge}} + 3\,\mathcal L_{\mathrm{recon}},$
where $\mathcal L_{\mathrm{CE}}$ is cross-entropy with label smoothing $0.1$, $\mathcal L_{\mathrm{hinge}}$ is the multiclass margin-hinge loss with margin $1$, and $\mathcal L_{\mathrm{recon}}$ is mean-squared reconstruction loss. Source training augmentation uses random crop with padding $4$ and random rotation by $10^\circ$.

\paragraph{CIFAR-10}
We use a convolutional autoencoder-based classifier. The encoder consists of convolutional blocks
$3 \rightarrow 32 \rightarrow 64 \rightarrow 128$, with batch normalization and ReLU activations, followed by fully connected layers
$128\cdot 8\cdot 8 \rightarrow 512 \rightarrow 128$ to produce a $128$-dimensional latent representation. The decoder mirrors this structure through
$128 \rightarrow 512 \rightarrow 128\cdot 8\cdot 8$ followed by transposed-convolution upsampling back to the input resolution. A spectral-normalized linear head maps the latent representation to $10$ logits. Training uses Adam, batch size $128$, $80$ epochs, learning rate $10^{-3}$, and weight decay $10^{-4}$. The loss is
$\mathcal L = \mathcal L_{\mathrm{CE}} + \mathcal L_{\mathrm{hinge}} + 0.1\,\mathcal L_{\mathrm{recon}},$
with standard cross-entropy, multiclass margin-hinge loss (margin $1$), and mean-squared reconstruction loss.

\paragraph{CIFAR-100}
\AQrev{We use a convolutional autoencoder-based classifier with the CIFAR-10 encoder blocks
$3 \rightarrow 32 \rightarrow 64 \rightarrow 128$ (stride $1$, $2$, $2$), batch normalization and ReLU activations, followed by fully connected layers
$128\cdot 8\cdot 8 \rightarrow 512 \rightarrow 256$ to produce a $256$-dimensional latent representation. The decoder mirrors the encoder. A spectral-normalized linear head maps the latent representation to $100$ logits. Training uses Adam, batch size $128$, $80$ epochs, learning rate $10^{-3}$, and weight decay $10^{-4}$. The loss is
$\mathcal L = \mathcal L_{\mathrm{CE}} + \mathcal L_{\mathrm{hinge}} + 0.1\,\mathcal L_{\mathrm{recon}},$
as for CIFAR-10. Source training augmentation uses random crop with padding $4$ and random horizontal flip.}

\begin{figure}[t]
\centering
\begin{subfigure}[b]{0.49\linewidth}
\centering
\includegraphics[width=\linewidth]{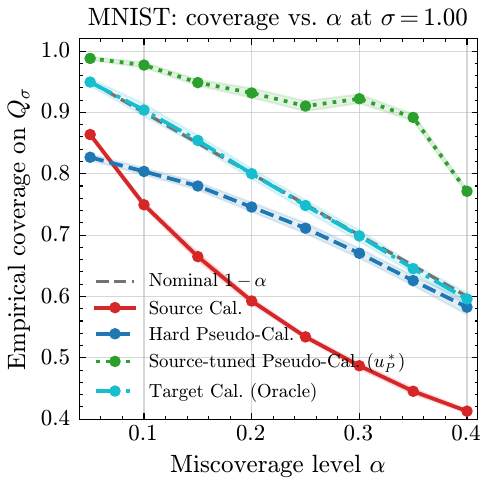}
\caption{MNIST}
\end{subfigure}
\hfill
\begin{subfigure}[b]{0.49\linewidth}
\centering
\includegraphics[width=\linewidth]{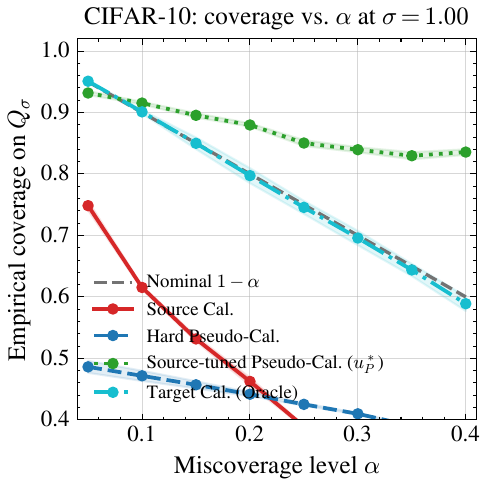}
\caption{CIFAR-10}
\end{subfigure}
\caption{Empirical coverage on $Q_\sigma$ versus the nominal miscoverage level $\alpha$ for representative shifted settings. The dashed line denotes the nominal target $1-\alpha$. Across the tested $\alpha$ values, source-tuned pseudo-calibration remains less prone to undercoverage than hard pseudo-calibration, but is often conservative for \AQrev{larger} $\alpha$.}
\label{fig:app_alpha_cov}
\end{figure}

\begin{figure}[t]
\centering
\begin{subfigure}[b]{0.49\linewidth}
\centering
\includegraphics[width=\linewidth]{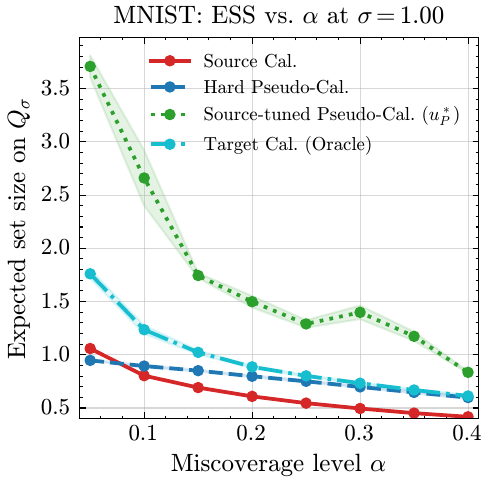}
\caption{MNIST}
\end{subfigure}
\hfill
\begin{subfigure}[b]{0.49\linewidth}
\centering
\includegraphics[width=\linewidth]{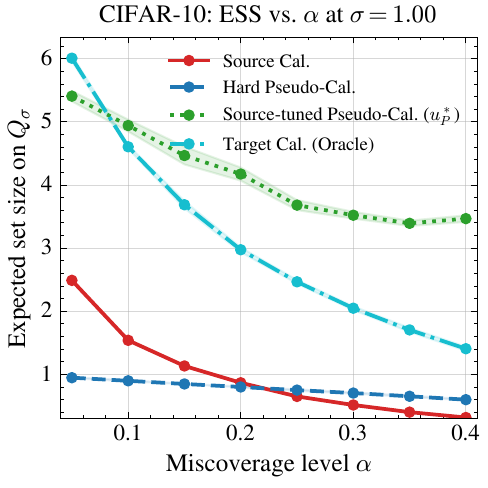}
\caption{CIFAR-10}
\end{subfigure}
\caption{Expected set size (ESS) on $Q_\sigma$ versus the nominal miscoverage level $\alpha$ for representative shifted settings. The improved coverage of source-tuned pseudo-calibration is accompanied by larger ESS, especially for smaller $\alpha$, illustrating the same coverage--set-size tradeoff seen in the main text.}
\label{fig:app_alpha_ess}
\end{figure}

\subsection{Coverage Difference vs.\ Classifier Loss and Shift Level}
\label{app:covdiff}

Here we report an additional experiment that examines how pseudo-calibration compares with source calibration as the source classifier improves. For each dataset (MNIST and CIFAR-10) and each shift level $\sigma$, we estimate the classifier's loss on the source test split $\mathcal{D}^P_{\mathrm{tst}}$, compute conformal thresholds for source calibration on $\mathcal{D}^P_{\mathrm{cal}}$ and hard pseudo-calibration on $\mathcal{D}^{Q_{\sigma}}_\mathrm{cal}$, and measure the difference in empirical coverage on $\mathcal{D}^{Q_{\sigma}}_\mathrm{tst}$. 
Fig.~\ref{fig:covdiff_mnist_cifar} shows the resulting coverage differences as a function of classifier loss and~$\sigma$. On both datasets, pseudo-calibration tends to achieve higher coverage than source calibration when $L_h(f,P)$ is relatively small. As $L_h(f,P)$ decreases, the range of shift levels $\sigma$ over which pseudo-calibration outperforms source calibration widens, empirically supporting the dependence on the source loss in our theoretical bounds.

\subsection{Sweeps over the Nominal Miscoverage Level}
\label{app:alpha_sweep}

We additionally examine how the methods behave as the nominal miscoverage level $\alpha$ varies. We report these sweeps here at the representative shift level $\sigma=1.0$.

Fig.~\ref{fig:app_alpha_cov} shows empirical coverage as a function of $\alpha$, together with the nominal target line $1-\alpha$. In both datasets, the same qualitative trend as in the main paper persists across the tested $\alpha$ values: hard pseudo-calibration is more prone to undercoverage than source-tuned pseudo-calibration, while the source-tuned method consistently yields higher target-domain coverage. At the same time, this improvement is conservative, especially for \AQrev{larger} $\alpha$, where source-tuned pseudo-calibration tends to overcover relative to the nominal level.

Fig.~\ref{fig:app_alpha_ess} shows the corresponding expected set size (ESS). The improved coverage of the source-tuned method is accompanied by larger ESS, particularly at smaller $\alpha$. As $\alpha$ increases, the ESS gap narrows, but the qualitative tradeoff remains the same. These additional results therefore reinforce the main message of the paper: source-tuning mitigates undercoverage under shift, at the cost of larger prediction sets.

\begin{figure}[t]
\centering
\begin{subfigure}[b]{0.49\linewidth}
\centering
\includegraphics[width=\linewidth]{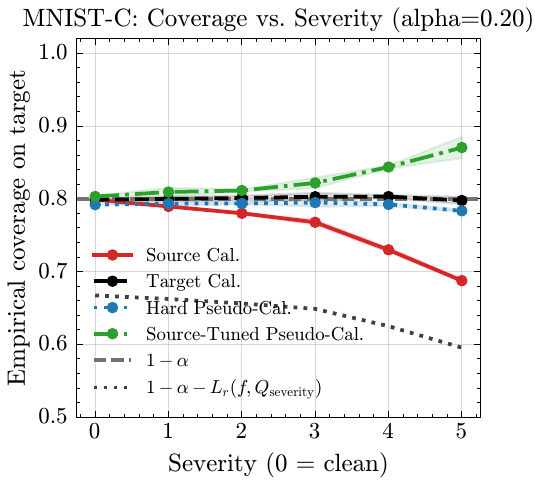}
\caption{MNIST-C}
\end{subfigure}
\hfill
\begin{subfigure}[b]{0.49\linewidth}
\centering
\includegraphics[width=\linewidth]{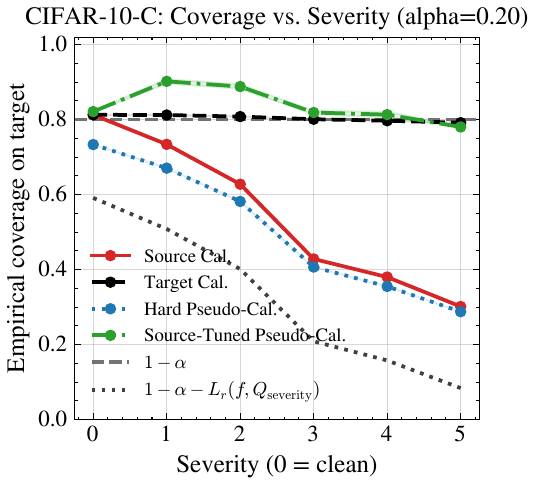}
\caption{CIFAR-10-C}
\end{subfigure}
\caption{Empirical coverage versus corruption severity on public corruption benchmarks at $\alpha=0.2$. Source-tuned pseudo-calibration remains less prone to undercoverage than hard pseudo-calibration across the tested severities.}
\label{fig:app_corruption_cov}
\end{figure}

\begin{figure}[t]
\centering
\begin{subfigure}[b]{0.49\linewidth}
\centering
\includegraphics[width=\linewidth]{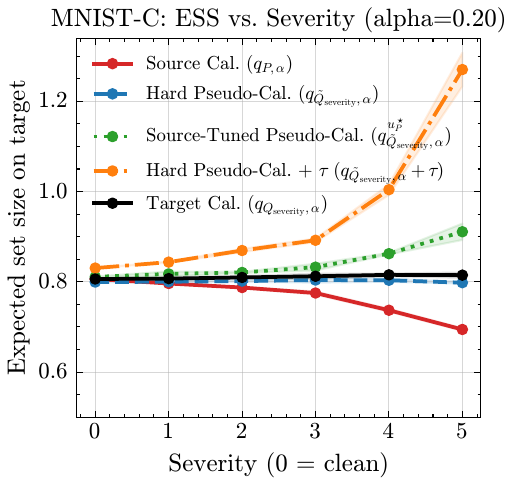}
\caption{MNIST-C}
\end{subfigure}
\hfill
\begin{subfigure}[b]{0.49\linewidth}
\centering
\includegraphics[width=\linewidth]{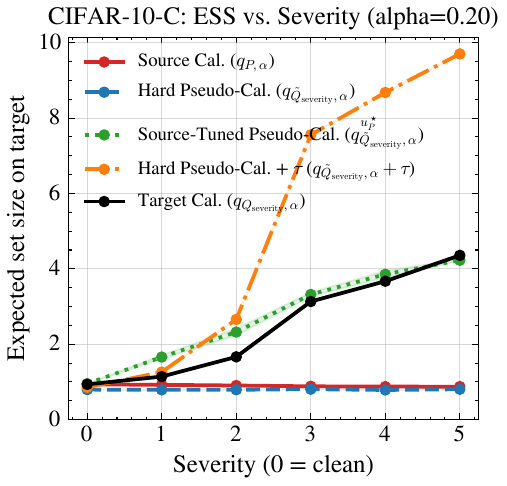}
\caption{CIFAR-10-C}
\end{subfigure}
\caption{Expected set size versus corruption severity on public corruption benchmarks at $\alpha=0.2$. The improved coverage of source-tuned pseudo-calibration is accompanied by larger expected set size, especially on CIFAR-10-C at higher severity.}
\label{fig:app_corruption_ess}
\end{figure}

\subsection{Public Corruption Benchmarks}
\label{app:corruption_benchmarks}

We evaluate the methods on public corruption benchmarks. We use MNIST-C and CIFAR-10-C with the \texttt{shot\_noise} corruption and severity levels $0,\dots,5$, where severity $0$ denotes the clean setting. \AQrev{The released MNIST-C data contain a single severity per corruption, so for MNIST we apply the \texttt{shot\_noise} function of the MNIST-C code~\cite{mu2019mnist} at its five severity levels.} We keep the source-trained classifier fixed, use the clean split as the source domain, and treat the corrupted split at each severity level as the target domain. Throughout these experiments, we use $\alpha=0.2$.

Fig.~\ref{fig:app_corruption_cov} shows empirical coverage versus corruption severity. On MNIST-C, source-tuned pseudo-calibration \AQrev{stays at or above} oracle calibration and above source calibration across the tested severity levels. On CIFAR-10-C, the same qualitative trend is even clearer: source-tuned pseudo-calibration is substantially less prone to undercoverage than hard pseudo-calibration and source calibration as severity increases, while remaining closest to oracle calibration among the unlabeled-target methods.

Fig.~\ref{fig:app_corruption_ess} shows the corresponding expected set sizes. As in the main paper, the improved coverage of the source-tuned method is accompanied by larger prediction sets, especially on CIFAR-10-C at higher corruption severities. These public-benchmark results therefore reinforce the central empirical message of the paper: source-tuning mitigates undercoverage under target shift, but does so conservatively.

\AQrev{
\subsection{Comparator Implementations}
\label{app:comparators}

\paragraph{ECP} Let $p_y(x)$ denote the softmax probability of label $y$. Following~\cite{kasa2025adaptingpredictionsetsdistribution}, the threshold $t_P$ is the empirical $\alpha(1+1/n)$ quantile of $p_Y(X)$ on labeled source calibration data, and $h$ is the $(1-\alpha)$ quantile of the predictive entropy on the target test inputs. A label is included when $p_y(x)\max(1,h)^2\ge t_P$, which is the default squared scaling of the official code. The linear scaling $p_y(x)\max(1,h)\ge t_P$ gives lower coverage than the squared one in every setting of Table~\ref{tab:all_cov_ess} (e.g., $55.38\%$ versus $60.51\%$ on CIFAR-10 at $\sigma=0.9$). As a check of the implementation, a pretrained ResNet-50 on ImageNet-R at $90\%$ nominal coverage gives $60.91\%$ coverage with ESS $9.19$ (linear) and $72.15\%$ with ESS $23.13$ (squared), compared with $61\%/9.1$ and $72\%/23.3$ reported in~\cite[Table~1]{kasa2025adaptingpredictionsetsdistribution}.

\paragraph{WQLCP} We follow the procedure of~\cite{Alijani_2025_CVPR}, with a variational autoencoder of our own design trained on source images to provide the reconstruction loss. It has two convolutional encoder stages with $32$ and $64$ channels, hidden width $256$ and latent dimension $32$, and is trained for $100$ epochs with AdamW (learning rate $10^{-4}$, batch size $256$, weight decay $0.01$, KL weight $1.2$). The source scores $1-p_Y(X)$ are weighted by their reconstruction losses, and $q_w$ is their weighted $(1-\alpha)$ quantile. With $r$ the $(1-\alpha)$ quantile of the target reconstruction loss, a label is included when $p_y(x)\max(1,r)\ge 1-q_w$. On ImageNet, the autoencoder has five stages with $16$, $32$, $64$, $128$ and $128$ channels and is trained on $5000$ held-out source images.

Both comparators use the softmax probability score of their original papers, whereas source, hard pseudo-, source-tuned and target calibration use the negative margin.

\subsection{All Methods on the Retrained Classifiers}
\label{app:retrained}

ECP and WQLCP in Table~\ref{tab:all_cov_ess} are evaluated on five classifiers per dataset retrained with the same code, seeds and hyperparameters as the other methods. The retrained MNIST classifiers reproduce Table~\ref{tab:all_cov_ess} exactly. On CIFAR, where GPU training is not bit-for-bit reproducible, the margin-based methods differ by at most about three points (Table~\ref{tab:retrained}). With all methods evaluated on the retrained classifiers, the best method under the rule of Table~\ref{tab:all_cov_ess} is the same in every setting. Across the five classifiers, the standard deviation of coverage is at most $4.08$ points for STPC (MNIST, $\sigma=2.0$) and at most $3.10$ points for the other methods.

\begin{table}[!htb]\centering\footnotesize
\caption{Margin-based methods on the five retrained CIFAR classifiers per dataset, with the protocol of Table~\ref{tab:all_cov_ess} ($1-\alpha=0.8$, five-run means). ECP and WQLCP in Table~\ref{tab:all_cov_ess} use these classifiers, and the retrained MNIST classifiers reproduce Table~\ref{tab:all_cov_ess} exactly.}\label{tab:retrained}
\setlength{\tabcolsep}{2.4pt}
\begin{tabular}{lr|cc|cc|cc|cc}
\toprule
 & & \multicolumn{2}{c|}{Source} & \multicolumn{2}{c|}{Hard} & \multicolumn{2}{c|}{STPC} & \multicolumn{2}{c}{Target}\\
 & $\sigma$ & Cov. & ESS & Cov. & ESS & Cov. & ESS & Cov. & ESS\\
\midrule
CIFAR-10 & 0.3 & 75.60 & 0.92 & 69.82 & 0.80 & 89.54 & 1.46 & 79.79 & 1.03\\
 & 0.9 & 51.02 & 0.87 & 48.75 & 0.80 & 89.35 & 3.63 & 80.20 & 2.46\\
 & 1.5 & 30.14 & 0.85 & 28.95 & 0.80 & 75.28 & 4.43 & 79.73 & 5.03\\
\midrule
CIFAR-100 & 0.3 & 75.53 & 4.62 & 41.83 & 0.81 & 93.36 & 15.74 & 79.66 & 5.85\\
 & 0.5 & 67.02 & 5.10 & 33.72 & 0.80 & 92.04 & 23.38 & 80.18 & 10.48\\
 & 0.7 & 56.72 & 5.20 & 26.38 & 0.81 & 86.68 & 25.74 & 79.64 & 17.06\\
\bottomrule
\end{tabular}
\end{table}

\subsection{Experimental Protocol}
\label{app:protocol}

\paragraph{Shift} In normalized image coordinates, $Q_\sigma$ applies per-image and per-channel contrast drawn from $\mathrm{Unif}[1-0.8\sigma,1+0.8\sigma]$, brightness from $\mathrm{Unif}[-0.4\sigma,0.4\sigma]$, periodic integer translations within $\pm\lfloor 4\sigma\rfloor$ pixels per axis, and elementwise Gaussian noise with standard deviation $0.1\sigma$, without clipping. At $\sigma=0$ it is the identity. Because the noise is unbounded, Assumption~\ref{assump:shift}(ii) does not hold exactly, but the proof of Theorem~\ref{thm:pseudo-calibration} uses it only through $\rho_{\mathrm{mix}}$, which is finite. For the label-free bound in Fig.~\ref{fig:cov_vs_sigma_all}, $\rho_{\mathrm{mix}}$ is estimated by the mean distance between the latent features of clean and shifted source images, excluding the largest $10\%$, and $L_\gamma$ by the largest distance between two rows of the linear head.

\paragraph{STPC grid} The grid $\mathcal U$ has $30$ equally spaced values between the $5$th and $95$th percentiles of the source predictive entropy. The fallback of Algorithm~\ref{alg:source-tuned-pseudo} is never used at $\alpha=0.2$ in any of the $45$ classifier/shift combinations.

\paragraph{Independent roles} In Table~\ref{tab:all_cov_ess}, the shifted target images are generated from the same source images that are used for tuning $u$. Repeating the comparison with disjoint source calibration, source tuning, target calibration and target test sets ($2500$ each) changes the coverage of STPC by at most $7.70$ points (MNIST, $\sigma=2.0$, below the target in both cases) and that of every other method by at most $0.90$ points, and the best method under the rule of Table~\ref{tab:all_cov_ess} is the same in every setting. The matched-coverage comparison and the calibration-averaged rescued mass below use these disjoint roles.

\subsection{Expected Set Size of Hard Pseudo-Calibration}
\label{app:hard_ess}

Let $\tilde S_i:=s(X_i,f(X_i))$ for $i\le n+1$ and $k:=\lceil(n+1)(1-\alpha)\rceil\le n$, so that $q_{\tilde Q,\alpha}$ is the $k$th smallest of $\tilde S_1,\ldots,\tilde S_n$. Barring logit ties, $\tilde S_i<0$ and $s(X_{n+1},y)>0$ for $y\neq f(X_{n+1})$, so $C^{1-\alpha}_{\tilde Q}(X_{n+1})$ is $\{f(X_{n+1})\}$ or empty, and its size is $\mathbbm 1\{\tilde S_{n+1}\le q_{\tilde Q,\alpha}\}$. The pseudo-scores do not involve labels and are i.i.d.\ under $Q_X$. If their distribution is continuous, the rank of $\tilde S_{n+1}$ among $\tilde S_1,\ldots,\tilde S_{n+1}$ is uniform on $\{1,\ldots,n+1\}$, and the expected set size is $\mathbb P_Q(\tilde S_{n+1}\le q_{\tilde Q,\alpha})=k/(n+1)\in[1-\alpha,\,1-\alpha+1/(n+1))$ for every target distribution. With ties, it is at least $k/(n+1)$. Conditional on the calibration sample, the expected set size is $F(q_{\tilde Q,\alpha})$, where $F$ is the CDF of $\tilde S_{n+1}$, and it follows a $\mathrm{Beta}(k,n+1-k)$ distribution with standard deviation close to $\sqrt{\alpha(1-\alpha)/n}$.

A fixed expected set size does not imply a fixed coverage. Since $Y_{n+1}\in C^{1-\alpha}_{\tilde Q}(X_{n+1})$ exactly when the pseudo-label is accepted and correct,
\[
\mathbb P_Q\big(Y_{n+1}\in C^{1-\alpha}_{\tilde Q}(X_{n+1})\big)=\tfrac{k}{n+1}-\mathbb P_Q(E),
\]
where $E:=\{f(X_{n+1})\neq Y_{n+1},\ \tilde S_{n+1}\le q_{\tilde Q,\alpha}\}$ is the event that a wrong pseudo-label is accepted. With the target error $e_Q:=Q\{f(X)\neq Y\}$, the coverage therefore lies between $\max\{0,k/(n+1)-e_Q\}$ and $\min\{k/(n+1),1-e_Q\}$. Since $e_Q\le L_r(f,Q)$, the lower end gives the second statement of Theorem~\ref{thm:pseudo-calibration}, which thus assumes that every misclassified input is accepted. The argument does not extend to STPC, whose randomized labels can produce sets with more than one label.

In our experiments, $n=5000$ and $k=4001$, so $k/(n+1)=0.8000$. On the retrained classifiers of Table~\ref{tab:retrained}, across all $41$ shift levels and five runs, no hard pseudo-calibrated set contains more than one label, and the per-run ESS has mean $0.8004$ and standard deviation $0.0072$, close to the predicted $0.0080=\sqrt{2\alpha(1-\alpha)/n}$, which adds the sampling noise of the $5000$ test points. The values $0.79$--$0.81$ in Table~\ref{tab:all_cov_ess} reflect this variation. Coverage is at most $\min\{\text{ESS},1-e_Q\}$ in every run, and between $11\%$ and $77\%$ of the misclassified inputs are accepted. For CIFAR-100 at $\sigma=0.3$, for example, the accuracy is $45.93\%$, $72\%$ of the errors are accepted, and the coverage is $41.83\%$, compared with the lower end of $25.93\%$.

\subsection{Rescued and Missed Mass}
\label{app:mass}

The proof of Theorem~\ref{thm:relative_coverage} gives the identity $F_{s\#Q}(t)=F_u(t)+r_u(t)-m(t)$, where $F_u(t)$ is the coverage of the mixed pseudo-labels, $r_u(t)$ is the rescued mass and $m(t)=Q\{s(X,f(X))\le t<s(X,Y)\}$ is the missed mass, i.e., the probability that the predicted label is in the set but the true label is not. Table~\ref{tab:mass} evaluates the three terms for STPC on the retrained classifiers of Table~\ref{tab:retrained}. The pseudo-label coverage is close to $80\%$ in every setting.
Since $r_u(t)=F_0(t)-F_u(t)$ with $F_0(t)=Q_X\{s(X,f(X))\le t\}$, the rescued mass is about $1-F_u\approx20\%$ whenever the STPC threshold is nonnegative, because every predicted label is then in the set. This holds in all settings except MNIST at $\sigma=0.7$, so the missed mass decides whether STPC reaches the target. For example, the net mass is $+13.38$ points for CIFAR-100 at $\sigma=0.3$ and $-27.36$ points for MNIST at $\sigma=2.0$. The rescued mass uses no target labels, while the missed mass and the coverage use target labels only for this evaluation.

\begin{table}[!htb]\centering\small
\caption{Rescued and missed mass of STPC ($1-\alpha=0.8$, five-run means, in \%). True-label coverage $=$ pseudo-label coverage $+$ rescued $-$ missed. Bold: positive net mass.}\label{tab:mass}
\setlength{\tabcolsep}{3pt}
\begin{tabular}{lr|cc|c|cc}
\toprule
Dataset & $\sigma$ & \shortstack{Pseudo-label\\coverage} & Rescued & Missed & Net & \shortstack{True-label\\coverage}\\
\midrule
MNIST & 0.7 & 79.97 & 15.03 & 4.86 & \textbf{+10.17} & 90.14\\
 & 1.6 & 80.00 & 20.00 & 24.28 & -4.28 & 75.72\\
 & 2.0 & 79.91 & 20.09 & 47.46 & -27.36 & 52.54\\
\midrule
CIFAR-10 & 0.3 & 79.70 & 20.30 & 10.46 & \textbf{+9.84} & 89.54\\
 & 0.9 & 79.66 & 20.34 & 10.65 & \textbf{+9.69} & 89.35\\
 & 1.5 & 79.64 & 20.36 & 24.72 & -4.37 & 75.28\\
\midrule
CIFAR-100 & 0.3 & 79.99 & 20.01 & 6.64 & \textbf{+13.38} & 93.36\\
 & 0.5 & 79.75 & 20.25 & 7.96 & \textbf{+12.29} & 92.04\\
 & 0.7 & 79.81 & 20.19 & 13.32 & \textbf{+6.87} & 86.68\\
\bottomrule
\end{tabular}
\end{table}

Theorem~\ref{thm:relative_coverage}(ii) replaces the missed mass by its upper bound $L_r(f,Q)$, which makes it conservative. For STPC, the missed mass is much smaller than the ramp loss (e.g., $6.64\%$ versus $63.26\%$ for CIFAR-100 at $\sigma=0.3$), because the ramp loss counts every input with margin below one, while the missed mass counts only inputs whose true label falls outside the set. Estimating the calibration-averaged $R_u$ by redrawing the target calibration sample $1000$ times (with replacement) and the auxiliary labels, for each fixed classifier and selected $u^\star$, gives $\widehat R$ between $6.93\%$ and $20.36\%$ across the nine settings. The rescued term thus raises the lower bound of Theorem~\ref{thm:pseudo-calibration} by up to $20$ points (e.g., from $16.74\%$ to $36.31\%$ for CIFAR-100 at $\sigma=0.3$).

\subsection{Coverage and Set Size at Matched Coverage}
\label{app:matched}

All methods build their sets by thresholding a fixed score. STPC and target calibration threshold the negative margin, and ECP and WQLCP threshold the softmax probability scaled by a single factor. Two methods that threshold the same score produce identical sets at the same threshold, so they trace the same nested family of sets, regardless of how their thresholds were selected. A comparison at matched coverage thus mainly reflects the score rather than the threshold selection, and it requires target labels. For each run we sweep $\alpha$ over a dense grid and interpolate ESS at $80\%$ target coverage. Table~\ref{tab:matched} shows that, at matched coverage, STPC agrees with target calibration to within $0.01$ in ESS, and that ECP and WQLCP agree with target calibration using the softmax score to within $0.04$.

\begin{table}[!htb]\centering\small
\caption{ESS at matched $80\%$ target coverage (independent roles, interpolated within each run, five-run means). Methods are grouped by score, and Target Cal.\ is labeled-target calibration with that score. Parentheses: runs whose coverage curve brackets $80\%$.}\label{tab:matched}
\setlength{\tabcolsep}{3pt}
\begin{tabular}{lr|cc|ccc}
\toprule
 & & \multicolumn{2}{c|}{Negative margin} & \multicolumn{3}{c}{Softmax probability}\\
Dataset & $\sigma$ & STPC & Target Cal. & ECP & WQLCP & Target Cal.\\
\midrule
MNIST & 0.7 & 0.82 & 0.82 & 0.84 & 0.84 (4/5) & 0.84\\
 & 1.6 & 3.11 & 3.11 & 3.12 & 3.12 & 3.12\\
 & 2.0 & 5.37 & 5.36 & 5.37 & 5.37 & 5.37\\
\midrule
CIFAR-10 & 0.3 & 1.03 & 1.03 & 1.02 & --- & 1.02\\
 & 0.9 & 2.39 & 2.39 & 2.33 & 2.33 & 2.33\\
 & 1.5 & 5.10 & 5.10 & 5.04 & 5.04 & 5.04\\
\midrule
CIFAR-100 & 0.3 & 6.01 & 6.01 & 5.44 & 5.44 & 5.44\\
 & 0.5 & 10.44 & 10.45 & 9.41 & 9.39 & 9.41\\
 & 0.7 & 17.34 & 17.34 & 15.91 & 15.88 & 15.92\\
\bottomrule
\end{tabular}
\end{table}

\subsection{ImageNet}
\label{app:imagenet}

We evaluate a pretrained ResNet-50 (torchvision IMAGENET1K\_V1) on ImageNet-C (Gaussian noise, defocus blur and fog, severities $1$--$5$) and ImageNet-R over five random splits. The clean validation images provide $25000$ source calibration, $5000$ source tuning and $15000$ source test images, and $5000$ further images train the WQLCP autoencoder. The corrupted source test images are split equally into target calibration and test sets. For ImageNet-R, the margin-based methods use the $200$ ImageNet-R classes, while ECP follows its official code. Table~\ref{tab:imagenet} shows that STPC reaches the coverage target with very large sets. With $1000$ classes and about $60\%$ target accuracy even at severity~$1$, many calibration inputs are randomized, and a uniformly random label is typically ranked far below the top, so the threshold and the sets become large. Randomizing over a smaller, classifier-informed label set is a natural extension.

\begin{table}[!htb]\centering\small
\caption{Pretrained ResNet-50 on ImageNet-C (15 corruption/severity conditions, averaged) and ImageNet-R (five-split means). Coverage in \%.}\label{tab:imagenet}
\setlength{\tabcolsep}{3pt}
\begin{tabular}{l|cc|cc}
\toprule
 & \multicolumn{2}{c|}{ImageNet-C} & \multicolumn{2}{c}{ImageNet-R}\\
Method ($1-\alpha=0.9$) & Cov. & ESS & Cov. & ESS\\
\midrule
Source Cal. & 63.37 & 10.50 & 31.24 & 0.62\\
Hard Pseudo-Cal. & 36.86 & 0.90 & 35.19 & 0.90\\
STPC (proposed) & 99.22 & 507.98 & 97.16 & 131.39\\
ECP & 82.40 & 44.54 & 72.15 & 23.13\\
WQLCP & 99.21 & 469.18 & 95.45 & 115.63\\
Target Cal. & 89.80 & 144.86 & 89.64 & 80.44\\
\bottomrule
\end{tabular}
\end{table}

\subsection{Natural Distribution Shifts}
\label{app:natural}

We evaluate ten CIFAR-10 classifiers, trained with the pipeline of Appendix~\ref{app:retrained} and used without any change, on CIFAR-10.1 (v6)~\cite{recht2018cifar} and CIFAR-10.2~\cite{lu2020harder}, each with $2000$ class-balanced test images split into $1000$ calibration and $1000$ test images per run. Source roles are disjoint as above. Table~\ref{tab:natural_cifar} shows that STPC reaches the $80\%$ target on both datasets with the smallest sets among the methods that do so. ECP has smaller sets but falls below the target, and WQLCP covers with larger sets.

\begin{table}[!htb]\centering\small
\caption{CIFAR-10 classifiers on CIFAR-10.1 and CIFAR-10.2 ($1-\alpha=0.8$, ten-run means). Boldface follows the rule of Table~\ref{tab:all_cov_ess}.}\label{tab:natural_cifar}
\setlength{\tabcolsep}{4pt}
\begin{tabular}{l|cc|cc}
\toprule
 & \multicolumn{2}{c|}{CIFAR-10.1} & \multicolumn{2}{c}{CIFAR-10.2}\\
Method & Cov.\ (\%) & ESS & Cov.\ (\%) & ESS\\
\midrule
Source Cal. & 66.24 & 0.89 & 63.86 & 0.89\\
Hard Pseudo-Cal. & 62.50 & 0.80 & 60.46 & 0.81\\
STPC (proposed) & \textbf{91.83} & \textbf{2.32} & \textbf{87.89} & \textbf{2.14}\\
ECP & 71.34 & 1.01 & 68.26 & 0.99\\
WQLCP & 97.91 & 3.80 & 96.32 & 3.73\\
Target Cal. & 80.09 & 1.35 & 79.43 & 1.48\\
\bottomrule
\end{tabular}
\end{table}

}

\end{document}